\documentclass{article}

\usepackage[preprint]{neurips_2026}

\usepackage[utf8]{inputenc} 
\usepackage[T1]{fontenc}    
\usepackage{hyperref}       
\usepackage{url}            
\usepackage{booktabs}       
\usepackage{amsfonts}       
\usepackage{nicefrac}       
\usepackage{microtype}      
\usepackage[table]{xcolor}
 \usepackage{graphicx}
 \usepackage{booktabs}
\usepackage{multirow}
\usepackage{threeparttable}
\usepackage{makecell}
\usepackage{enumitem}
\usepackage{amsmath}
\usepackage{amsthm}
\usepackage[normalem]{ulem}
\usepackage{subcaption}
\usepackage[ruled,vlined,linesnumbered]{algorithm2e}
\usepackage{float}
\newtheorem{theorem}{Theorem}

\SetKwInOut{Input}{Input}
\SetKwInOut{Output}{Output}
\SetKw{KwRet}{return}
\DontPrintSemicolon
\SetAlFnt{\small}
\SetAlCapFnt{\small}
\SetAlCapNameFnt{\small}
\SetAlgoNlRelativeSize{-1}

\title{SWIFT: Prompt-Adaptive Memory for Efficient Interactive Long Video Generation}
\author{%
  \normalfont
  Shanwen Tan\textsuperscript{1} \quad
  Hao Li\textsuperscript{2} \quad
  Jingtao Zhang\textsuperscript{3} \quad
  Xiaosong Jia\textsuperscript{2} \\
  Xue Yang\textsuperscript{4} \quad
  Shaofeng Zhang\textsuperscript{1}\thanks{Corresponding author.} \quad
  Yanyong Zhang\textsuperscript{1} \\
  \textsuperscript{1}University of Science and Technology of China \quad
  \textsuperscript{2}Fudan University \\
  \textsuperscript{3}Georgia Institute of Technology \quad
  \textsuperscript{4}Shanghai Jiao Tong University
}
\begin{document}

\maketitle

\begin{abstract}
  Streaming long-video generation faces a central challenge in continuous semantic switching, requiring adaptive memory to preserve coherent visual evolution. Current approaches rely on cache rebuilding at prompt boundaries or fixed memory budgets, but they introduce redundant computation and limit flexible semantic adaptation. This limitation arises from a mismatch between cached video history and prompt updates, as memory preserves visual continuity while prompt switches demand rapid semantic adaptation. Motivated by this observation, we present \textbf{SWIFT} (\underline{S}emantic \underline{W}indowing and \underline{I}njection for \underline{F}lexible \underline{T}ransitions), a training-free framework for multi-prompt long-video generation that enables efficient semantic switching while preserving temporal coherence in causal video diffusion models. SWIFT introduces a lightweight \emph{Semantic Injection Cache} that augments cached video memory rather than reconstructing it from scratch at every prompt boundary. To avoid uniformly perturbing all attention channels, we further perform head-wise semantic injection, so that each attention head receives a prompt update proportional to its alignment with the current video state. In addition, we introduce an \emph{Adaptive Dynamic Window} that allocates temporal memory according to prompt phase, using larger local context near switching boundaries and smaller windows during stable segments to reduce average inference cost. To preserve long-range semantic consistency under compressed local attention, we further maintain segment-level semantic anchors that summarize prompt-conditioned video history and reintroduce it as compact memory tokens. Compared with current state-of-the-art methods, SWIFT preserves generation quality while achieving 22.6 FPS on a single H100 GPU, establishing a substantially more efficient solution for multi-prompt long-video generation. Our code is available at \url{https://github.com/ShanwenTan/SWIFT}.

\end{abstract}

\vspace{-11pt}
\section{Introduction}
Video generation delivers strong visual realism and temporal stability \cite{qing2024hierarchical, brooks2022generating, yin2025survey, he2022latent}, with broad relevance to filmmaking \cite{yuan2024mora}, digital content creation \cite{ju2025fulldit}, and immersive media \cite{zhou2024upscale}. Autoregressive video diffusion \cite{gu2025long} extends this progress by rolling out frames sequentially for variable-length synthesis~\cite{huang2025self, zhou2024storydiffusion, yin2024one}. Representative efforts include constant memory inference in FIFO-Diffusion \cite{kim2024fifo}, training-free long-horizon extension in LongDiff \cite{li2025longdiff}, and supervised causal generation with autoregressive Transformers \cite{yin2025slow}. Beyond length extension, semantic control throughout generation is equally important, rather than only conditioning at initialization \cite{xie2025comprehensive, feng2025blobgen}.


Multi-prompt video generation is more challenging than single-prompt synthesis, as the model must adapt to prompt updates while ensuring visual coherence, stable transitions, and consistent dynamics. Existing memory-based pipelines struggle with balancing semantic adaptation and temporal preservation \cite{wu2025mind, liu2025worldweaver}. Several recent studies have started to explore multi-prompt long video generation through distinct control and memory paradigms \cite{cui2026selfforcing, wu2025video}. DiTCtrl \cite{cai2025ditctrl} addresses sequential prompt control through attention level modulation, while LongLive \cite{yang2026longlive} resolves prompt switching by recaching memory at semantic boundaries. SynCoS~\cite{kim2025tuning} extends pretrained text-to-video models for multi-event long video generation by synchronizing local smoothness and global coherence. However, limited responsiveness to newly activated prompts or the additional cost of memory reconstruction still constrains their scalability in long-horizon interactive generation \cite{lin2026autoregressive, xi2025sparse}. This exposes a central requirement for multi-prompt long-video generation: \textbf{\textit{an efficient memory mechanism must accommodate continuous prompt shifts while supporting coherent long-range generation.}}

\setlength{\textfloatsep}{3pt}
\begin{figure*}[t]
  \centering
  \includegraphics[width=\textwidth]{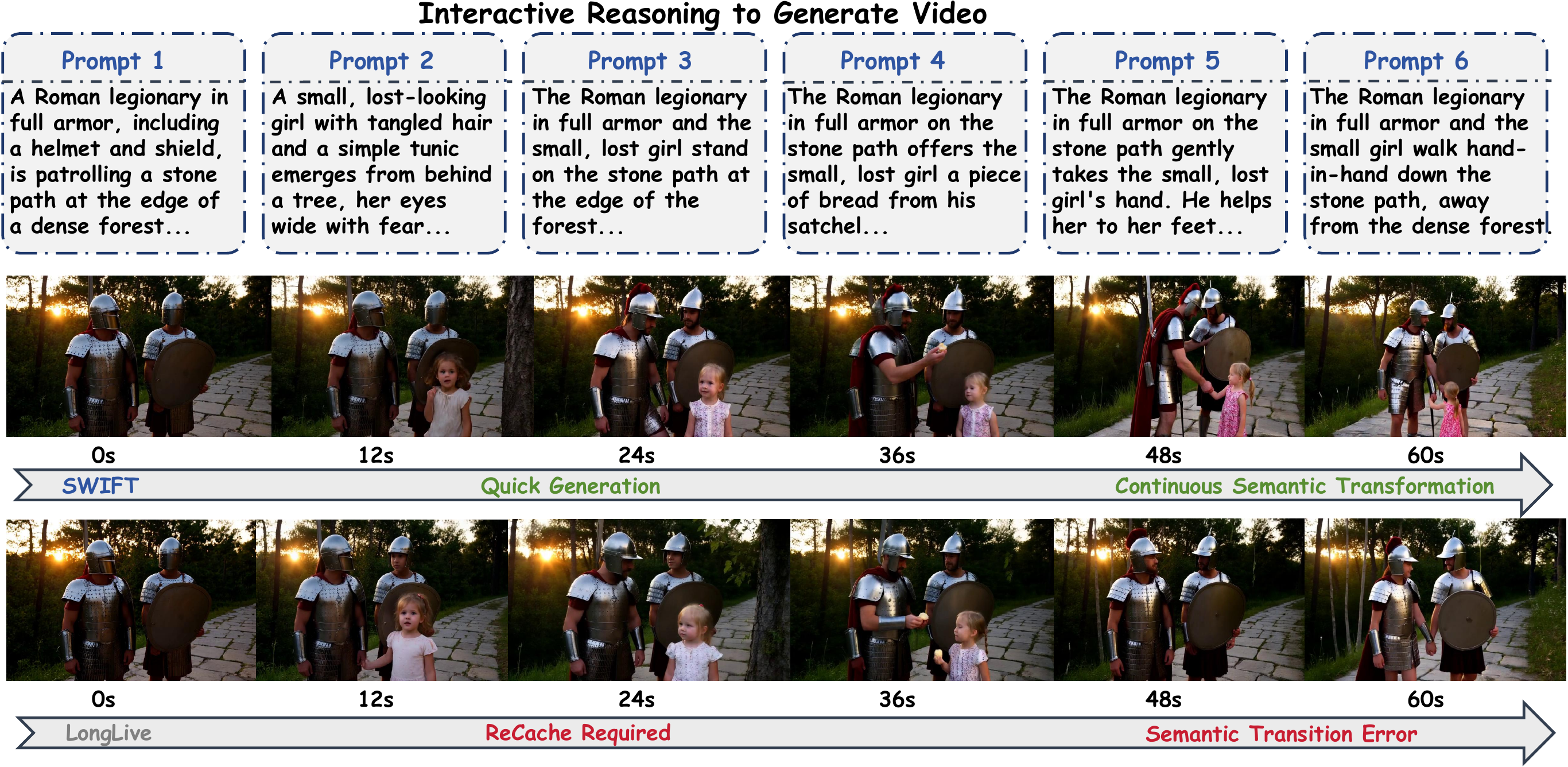}
\caption{
Streaming interactive T2V models, e.g., LongLive~\cite{yang2026longlive}, face unstable prompt transitions and recaching overhead. SWIFT addresses both with Semantic Injection Cache and Adaptive Dynamic Window, improving transition smoothness, long-range consistency, and inference efficiency.
}
  \label{fig:case}
\end{figure*}

Motivated by these challenges, we propose \textbf{SWIFT} 
(\underline{S}emantic \underline{W}indowing and 
\underline{I}njection for \underline{F}lexible 
\underline{T}ransitions), an inference-time memory framework for multi-prompt long-video generation. SWIFT enables efficient adaptation to continuous prompt transitions while preserving coherent long-range visual evolution (Figure \ref{fig:case}). SWIFT integrates two complementary components: (i) \textbf{\textit{Semantic Injection Cache}}, which injects transition-aware semantic signals into memory at prompt boundaries to improve adaptation without full cache reconstruction; and (ii) \textbf{\textit{Adaptive Dynamic Window}}, which allocates the temporal read budget according to the current generation phase to balance long-range coherence and inference efficiency. By combining Semantic Injection Cache with Adaptive Dynamic Window, SWIFT improves prompt responsiveness, visual coherence, and efficiency in multi-prompt long-video generation. The main contributions of this paper can be summarized in three parts.
\begin{itemize}[leftmargin=*, itemsep=4pt, topsep=4pt]
\item \textbf{New Perspective.} Multi-prompt long-video generation is studied from the perspective of continuous memory adaptation. Prompt switching is addressed through efficient historical-memory updating rather than repeated scratch reconstruction.

\item \textbf{Unified Method.} We propose SWIFT, an inference-time memory framework for multi-prompt long-video generation. SWIFT couples transition-aware semantic memory refinement with prompt-phase adaptive temporal budgeting, enabling efficient multi-prompt long-video generation with strong controllability and coherence.

\item \textbf{Empirical Validation.} Extensive experiments are conducted under multi-prompt long-video generation settings. SWIFT achieves strong visual quality and semantic alignment while reaching 22.6 FPS on a single H100 GPU under the multi-prompt setting.
\end{itemize}

\section{Related Works}
\noindent\textbf{Long Video Generation.}
Recent progress in video generation \cite{lee2024grid} has extended synthesis from short clips to longer and more coherent sequences through autoregressive diffusion \cite{yin2025slow, gao2025ca2}, diffusion forcing \cite{song2025history}, and memory-based streaming inference \cite{zhu2025memorize}. Autoregressive and hybrid autoregressive-diffusion methods generate long videos by iteratively rolling out future chunks while reusing historical context \cite{xie2025progressive}, achieving strong temporal continuity and efficient causal generation. Diffusion-forcing methods improve long-horizon rollout by aligning training and inference under partially generated contexts. VideoTetris~\cite{tian2024videotetris} enables compositional text-to-video generation by spatially and temporally composing attention maps to better follow complex textual semantics.

\noindent\textbf{Interactive and Multi-Prompt Video Generation.}
Interactive and multi-prompt video generation \cite{cai2025ditctrl, wu2025mind} requires the model to adapt to changing text conditions while preserving visual continuity across prompt boundaries. Existing approaches commonly decompose the task into multiple short clips \cite{villegas2022phenaki} and then apply stitching, interpolation, or segment-wise generation, which often weakens temporal coherence over long horizons. StoryMem~\cite{zhang2025storymem} further explores multi-shot long video storytelling by maintaining an explicit visual memory bank to improve cross-shot consistency across sequential video segments.

\noindent\textbf{Memory-Efficient Inference for Long-Horizon Generation.}
To improve the efficiency of long-horizon inference, existing methods commonly restrict attention to local windows \cite{xie2025progressive,yang2026longlive} or compact memory states \cite{cai2025mixture} instead of reading the full history densely. These designs provide an effective trade-off between efficiency and long-range context, although most of them rely on fixed memory budgets across different generation stages. SWIFT addresses this challenge with a continuous and dynamic memory mechanism for semantic transition and temporal memory allocation.

\vspace{-6pt}
\section{Method}
\vspace{-4pt}
\label{sec:method}
SWIFT (\underline{S}emantic \underline{W}indowing and \underline{I}njection for 
\underline{F}lexible \underline{T}ransitions) is an inference-time mechanism for efficient semantic transition and memory cache management in interactive long-video generation under sequential prompt control. Built upon causal autoregressive video diffusion, the method combines two coupled designs: a \emph{Semantic Injection Cache} (section \ref{subsec:semantic_injection_cache}) for memory updates at prompt transitions via lightweight semantic bridging rather than expensive cache reconstruction, and an \emph{Adaptive Dynamic Window} (section \ref{subsec:adaptive_dynamic_window}) for phase-aware control of the effective temporal memory span with segment-level semantic anchors augmenting compressed local context. These designs preserve temporal continuity, improve semantic responsiveness at prompt boundaries, and reduce the average attention cost during long-horizon rollout.

\begin{figure*}[t]
  \centering
  \includegraphics[width=\textwidth]{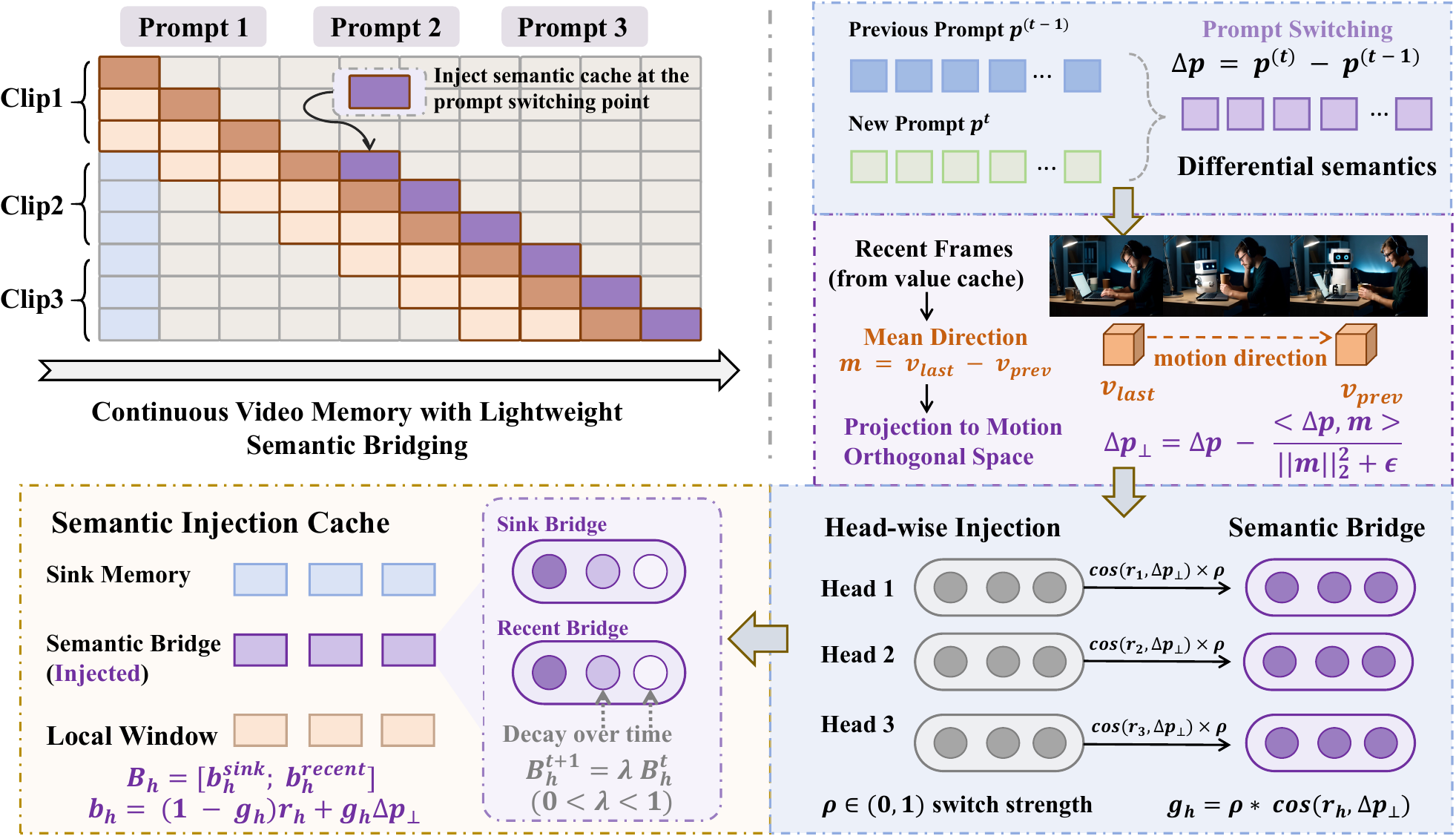}
  \caption{
  Illustration of \textbf{Semantic Injection Cache}. Instead of rebuilding the full video cache at every prompt boundary, SWIFT constructs a lightweight semantic bridge from the prompt transition signal. The transition is first projected onto a motion-orthogonal subspace to avoid interfering with local temporal dynamics, and is then injected into memory through head-wise alignment with recent and sink summaries. The injected bridge is read together with continuous video memory, providing efficient semantic switching while preserving motion continuity.
  }
  \label{fig:semantic_injection_cache}
\end{figure*}

\vspace{-4pt}
\subsection{Semantic Injection Cache}
\label{subsec:semantic_injection_cache}
\vspace{-4pt}

Prompt transitions introduce a mismatch between the newly activated prompt semantics and the historical video memory accumulated under earlier prompts \cite{yu2025context}. Existing methods typically address this issue by recomputing the cache at prompt boundaries \cite{yang2026longlive}, which improves semantic alignment but incurs substantial additional computation. We introduce a Semantic Injection Cache (Figure \ref{fig:semantic_injection_cache}) that updates memory through the detected semantic transition, enabling prompt adaptation without recomputing the full cache while preserving the visual history accumulated in prior memory states.

\textbf{Prompt Transition Signatures. }For each prompt segment, we first compute a projected prompt signature by aggregating the prompt embeddings after projection into the generator hidden space. Let \(p^{(m-1)}\) and \(p^{(m)}\) denote the projected signatures of the previous and current prompts, represented in the same head-wise value-cache space as the cached visual summaries. We define the prompt transition signal as:
\begin{equation}
\Delta p^{(m)} = p^{(m)} - p^{(m-1)}.
\label{eq:prompt_delta_method}
\end{equation}
This difference vector captures the semantic displacement induced by the prompt switch. We further quantify the switch magnitude through the cosine-based transition strength
\begin{equation}
\rho^{(m)} = 1 - \cos\!\left(p^{(m-1)},\, p^{(m)}\right),
\label{eq:switch_strength_method}
\end{equation}
which is later used to modulate the amount of semantic injection.
We decompose $\Delta p^{(m)}$ along the local cache tangent $m$. Under a
first-order cache response model, its orthogonal component is motion-neutral,
yielding the following projection:

\begin{theorem}[\textbf{Motion-neutral semantic projection}]
\label{thm:motion_neutral_projection}
Let $\mathcal{H}$ denote the value-cache representation space. At a prompt
switch, $\Delta p^{(m)}=p^{(m)}-p^{(m-1)}$ represents the prompt displacement,
while $m=v_t-v_{t-1}$ denotes the local cache tangent estimated from recent
value-cache summaries. Under a local first-order cache response assumption,
motion-neutral injection requires removing the component of $\Delta p^{(m)}$
parallel to $m$. The resulting update is the closest prompt-preserving vector in
the motion-neutral subspace:
\begin{equation}
\label{eq:motion_neutral_qp_main}
    \Delta p_{\perp}^{(m)}
    =
    \arg\min_{\Delta x\in\mathcal{H}}
    \|\Delta x-\Delta p^{(m)}\|_2^2
    \quad
    \mathrm{s.t.}\quad
    \langle \Delta x,m\rangle=0 ,
\end{equation}
which admits the closed-form solution
\begin{equation}
\label{eq:motion_orthogonal_projection_exact}
    \Delta p_{\perp}^{(m)}
    =
    \Delta p^{(m)}
    -
    \frac{\left\langle \Delta p^{(m)},m\right\rangle}
    {\|m\|_2^2}m .
\end{equation}
Proof is provided in Appendix~\ref{sec:appendix_theoretical_justification}.
Thus, semantic injection preserves the largest first-order motion-neutral prompt component while avoiding direct perturbation of short-term dynamics.
\end{theorem}
In practice, we use the stabilized form (where $\epsilon>0$ prevents numerical amplification when the estimated motion magnitude is small).
\begin{equation}
\Delta p_{\perp}^{(m)}
=
\Delta p^{(m)}
-
\frac{\left\langle \Delta p^{(m)},\, m \right\rangle}{\|m\|_2^2 + \epsilon}\, m.
\label{eq:motion_orthogonal_projection_method}
\end{equation}

\textbf{Head-wise Semantic Injection.}
We construct the semantic bridge head-wise to reflect heterogeneous memory
reading patterns across attention heads, from local dynamics to long-range
contextual anchors. For each layer, we extract two summaries from the value cache: a recent summary $r$ from the most recent frames and a sink summary $s$ from the persistent sink region. We measure their alignment with the projected semantic transition by clipped cosine similarity and use the switch strength \(\rho^{(m)}\) to define bounded head-wise gates:
\begin{equation}
g_r = \rho^{(m)} \left[\cos\!\left(r,\, \Delta p_{\perp}^{(m)}\right)\right]_{0}^{1},
\qquad
g_s = \rho^{(m)} \left[\cos\!\left(s,\, \Delta p_{\perp}^{(m)}\right)\right]_{0}^{1}.
\label{eq:headwise_gates_method}
\end{equation}
Here, \([\cdot]_{0}^{1}\) denotes clipping to \([0,1]\), so \(g_r\) and \(g_s\) act as bounded interpolation gates that control how much semantic transition is injected into the recent and sink-aligned summaries. We then construct two bridge components:
\begin{equation}
B_r = (1-g_r)\, r + g_r\, \Delta p_{\perp}^{(m)},
\qquad
B_s = (1-g_s)\, s + g_s\, \Delta p_{\perp}^{(m)}.
\label{eq:bridge_construction_method}
\end{equation}
The two bridge components are concatenated as a transient bridge memory \(B_h=[B_s;B_r]\), which is written into the bridge slots of the key-value attention cache. During self-attention, these bridge cache entries are read together with sink memory, segment anchors, and the adaptive local window, providing prompt-transition guidance without rebuilding the full historical cache. After each generated block, the bridge contribution is decayed by \(B_h^{t+1}=\lambda B_h^t\), allowing newly generated video evidence to gradually dominate once the prompt transition has been absorbed.

\begin{figure*}[t]
  \centering
  \includegraphics[width=\textwidth]{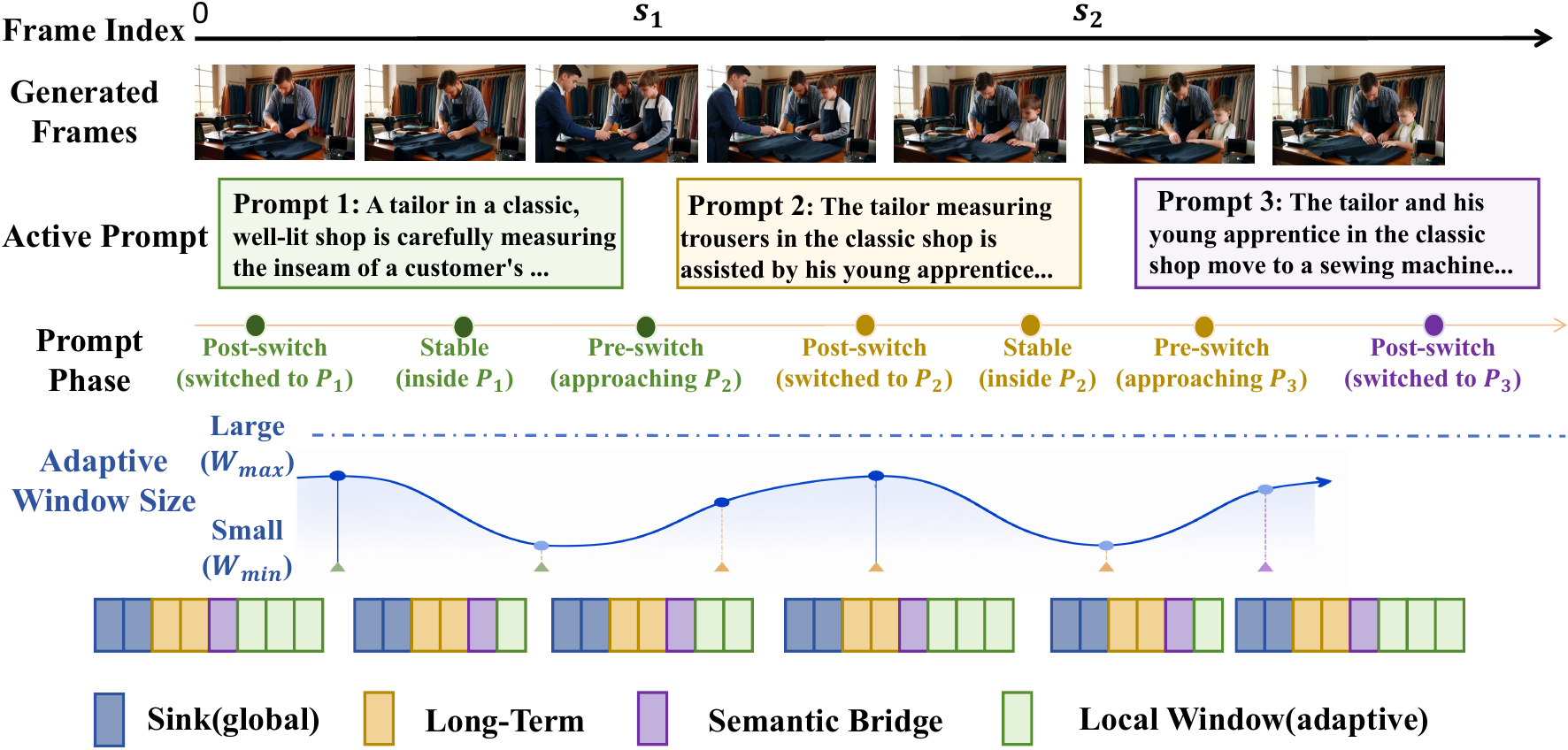}
  \caption{
  Illustration of \textbf{Adaptive Dynamic Window}. SWIFT allocates temporal memory according to prompt phase rather than using a fixed local attention span throughout generation. The effective window expands around prompt transitions for stable semantic handover and shrinks inside stable intervals for efficient rollout. Segment-level semantic anchors compensate for the reduced local context by preserving compact prompt-conditioned summaries of previous segments, thereby lowering average attention cost without sacrificing long-range coherence.
  }
  \label{fig:adaptive_dynamic_window}
\end{figure*}

\subsection{Adaptive Dynamic Window}
\label{subsec:adaptive_dynamic_window}
In long-video generation, the inference cost of causal attention increases with sequence length \cite{wang2025lingen, xi2025sparse}, while recent frames usually contribute more to the next prediction than distant history. Existing local-attention methods therefore use a fixed temporal window for efficiency \cite{zhang2025fast, yang2026longlive}, but different inference phases favor different window sizes. To address this issue, an adaptive dynamic window (Figure \ref{fig:adaptive_dynamic_window}) is introduced to allocate different temporal read budgets across multi-prompt generation phases, improving efficiency in semantically stable intervals and preserving sufficient context around prompt transitions.

\paragraph{Prompt-Phase Window Scheduling.}
Let $t$ denote the current block position, let $s_m$ denote the start of the active prompt segment, and let $s_{m+1}$ denote the next switch boundary. We define the segment age
\begin{equation}
a_t = \max(0,\, t - s_m),
\end{equation}
and the distance to the next switch
\begin{equation}
d_t =
\begin{cases}
\max(0,\, s_{m+1} - t), & \text{if the next switch exists},\\[3pt]
+\infty, & \text{otherwise}.
\end{cases}
\end{equation}
The prompt phase of the current generation step is determined by $a_t$ and $d_t$. Early after a prompt switch, the model requires a larger local window to maintain stable semantic transition. In the stable portion of a prompt segment, the local context demand becomes weaker and the window can be reduced. As generation approaches the next switch, the window expands again to prepare for the upcoming semantic transition.
To realize this phase-dependent scheduling, we define a post-switch decay factor and a pre-switch expansion factor
\begin{equation}
w_{\text{post}}(t)
=
\exp\!\left(-\frac{a_t}{\tau_{\text{post}}}\right), \quad
w_{\text{pre}}(t)
=
\exp\!\left(-\frac{d_t}{\tau_{\text{pre}}}\right).
\end{equation}
The effective phase weight is
\begin{equation}
w_t = \max\!\left\{ w_{\text{post}}(t),\; w_{\text{pre}}(t) \right\}.
\label{eq:phase_weight_method}
\end{equation}
Given a maximum window size $W_{\max}$ and a minimum window size $W_{\min}$, the adaptive window is defined as
\begin{equation}
W_t
=
W_{\min}
+
\left(W_{\max} - W_{\min}\right) w_t.
\label{eq:adaptive_window_method}
\end{equation}
The physical cache budget remains fixed, and the adaptive mechanism only modulates the effective temporal read range at each step. This design avoids cache reallocation overhead while reducing the average attention cost over long-horizon generation.

\paragraph{Segment-Level Semantic Anchors.}
A small local window improves efficiency, but it also reduces direct access to long-range semantic context \cite{zhang2025fast}. To compensate for this loss, we introduce segment-level semantic anchors. At the end of each prompt segment, we summarize the recent value cache and combine this visual summary with the projected prompt signature of the same segment. Let $u^{(m)}$ denote the summary of the recent value cache from the completed segment, and let $p^{(m)}$ denote the semantic prompt signature obtained by projecting the prompt into the generator hidden space. We define the segment anchor as
\begin{equation}
A^{(m)} = (1-\alpha)\, u^{(m)} + \alpha\, p^{(m)},
\label{eq:segment_anchor_method}
\end{equation}
where $\alpha$ controls the balance between visual memory and semantic identity.
This anchor provides a compact prompt-conditioned summary of the segment and is injected into attention together with sink memory, semantic bridge entries, and the current local window. Long-range semantic context is therefore preserved without incurring the full cost of dense historical attention.

\textbf{Unified Memory View. }Under the proposed design, the memory used by the generator is no longer a single fixed-size local cache. Instead, it becomes a structured memory system composed of continuous video history, transient semantic bridges, compact segment anchors, and a prompt-phase adaptive local window. The semantic injection cache governs how the model changes its interpretation of history at prompt transitions. The adaptive dynamic window governs where computation is spent over time. The segment anchors preserve long-range semantic context when the local window is compressed. Together, these components enable efficient and coherent multi-prompt long-video generation without relying on repeated full memory reconstruction. The detailed procedures for semantic injection and adaptive window scheduling are summarized in Algorithm~\ref{alg:semantic_injection_cache} and Algorithm~\ref{alg:adaptive_dynamic_window}.


\vspace{-4pt}
\section{Experiments}
\vspace{-3pt}
\subsection{Experimental Setup}
\textbf{Implementation Details.}
SWIFT is instantiated on top of Wan2.1-T2V-1.3B \cite{wan2025wan}, inherits the inference protocol of LongLive \cite{yang2026longlive}, and extends the backbone with our semantic injection and adaptive memory design. SWIFT is compatible with autoregressive video diffusion models that expose reusable attention memory value cache and support blockwise causal rollout. 

\textbf{Evaluation protocol.} Following the same setting as LongLive \cite{yang2026longlive} and Self-Forcing \cite{huang2025self}, we use Qwen2-72B-Instruct~\cite{yang2025qwen3} to generate the switched-prompt dataset. Each script set contains 6 consecutive 10-second cue sequences, for a total of 100 videos. Each video lasts 60 seconds.
We use VBench-Long \cite{huang2024vbench} to evaluate the visual quality of all generated videos, focusing on image quality, subject consistency and aesthetics, to compare long-range consistency and visual effects in terms of subject, background and visual aesthetics. For semantic alignment evaluation under sequential prompt control, each generated video is partitioned into 10-second segments according to the prompt schedule, and the CLIP \cite{radford2021learning} score is computed between each segment and its corresponding prompt.

\textbf{Baselines.} We compare SWIFT with several representative state-of-the-art long video generation methods under prompt switching at inference time. 
The compared methods include MemFlow~\cite{ji2025memflow}, FramePack~\cite{zhang2025packing}, Self Forcing~\cite{huang2025self}, CausVid~\cite{yin2025slow}, and LongLive~\cite{yang2026longlive}. 
Each script contains 6 consecutive 10-second prompts, yielding 100 generated videos totaling 60 seconds each. Appendix \ref{deatailed_setup} provides detailed information on SWIFT settings.

\begin{table}[t]
\centering
\caption{\textbf{CLIP score and FPS comparison under the multi-prompt 60-second setting.} CLIP scores are computed at prompt-aligned intervals.}
\label{tab:multi_prompt_60s}
\definecolor{lightgrayrow}{gray}{0.93}
\setlength{\tabcolsep}{2pt}
\renewcommand{\arraystretch}{1.05}
\resizebox{\columnwidth}{!}{%
\begin{tabular}{lcccccccc}
\toprule
\multirow{2}{*}{\textbf{Method}} 
& \multirow{2}{*}{\#Params}
& \multirow{2}{*}{\makecell{Throughput\\(FPS)$\uparrow$}}
& \multicolumn{6}{c}{CLIP Score$\uparrow$} \\
\cmidrule(lr){4-9}
& & & 0--10\,s & 10--20\,s & 20--30\,s & 30--40\,s & 40--50\,s & 50--60\,s \\
\midrule
MemFlow~\cite{ji2025memflow}    
& 1.3B & 18.7  
& 26.29 & 24.64 & 23.85 & 23.53 & 24.27 & \underline{23.73} \\
Self Forcing~\cite{huang2025self}  
& 1.3B & 17.0  
& 26.24 & 24.87 & 23.46 & 21.92 & 22.05 & 21.07 \\
LongLive~\cite{yang2026longlive}         
& 1.3B & \underline{20.7}  
& \textbf{26.63} & \underline{25.47} & \underline{24.65} & \underline{23.69} & \underline{24.52} & {23.61} \\
CausVid~\cite{yin2025slow}
& 1.3B & 17.0
& 26.55 & 24.93 & 23.82 & 22.74 & 23.16 & 22.68 \\
FramePack~\cite{zhang2025packing}       
& 13B & 6.7   
& 26.51 & 22.60 & 22.18 & 21.53 & 21.98 & 21.62 \\
\rowcolor{lightgrayrow}
\textbf{SWIFT}                
& 1.3B & \textbf{22.6} 
& \underline{26.53} & \textbf{25.86} & \textbf{24.71} & \textbf{24.04} & \textbf{24.68} & \textbf{24.13} \\
\bottomrule
\end{tabular}%
}
\end{table}
\begin{table}[t]
\centering
\caption{\textbf{Quantitative comparison under multi-prompt 60-second setting} with representative long video generation models.}
\label{tab:multi_prompt_quality}
\definecolor{lightgrayrow}{gray}{0.93}
\setlength{\tabcolsep}{2pt}
\renewcommand{\arraystretch}{1.05}
\resizebox{\columnwidth}{!}{%
\begin{tabular}{lccccccc}
\toprule
\textbf{Method}
& {Resolution}
& \makecell{Aesthetic\\Quality$\uparrow$}
& \makecell{Background\\Consistency$\uparrow$}
& \makecell{Imaging\\Quality$\uparrow$}
& \makecell{Motion\\Smoothness$\uparrow$}
& \makecell{Subject\\Consistency$\uparrow$}
& \makecell{Temporal\\Flickering$\uparrow$} \\
\midrule
MemFlow~\cite{ji2025memflow}    
& 832$\times$480 & \underline{60.02} & 96.21 & 71.69 & 98.83 & \underline{97.63} & 97.67 \\
Self Forcing~\cite{huang2025self}  
& 832$\times$480 & 58.45 & 95.74 & 71.33 & 98.56 & 95.35 & 97.52 \\
LongLive~\cite{yang2026longlive}         
& 832$\times$480 & {59.89} & 96.05 & \underline{71.74} & \underline{99.14} & 97.32 & \textbf{98.46} \\
CausVid~\cite{yin2025slow}
& 832$\times$480 & 59.77 & \underline{95.42} & 71.73 & 99.02 & {96.15} & 98.02 \\
FramePack~\cite{zhang2025packing}       
& 832$\times$480 & 59.68 & \textbf{96.77} & 71.71 & 98.87 & \textbf{97.86} & 97.41 \\
\rowcolor{lightgrayrow}
\textbf{SWIFT}                
& 832$\times$480 & \textbf{60.23} & \underline{96.28} & \textbf{72.32} & \textbf{99.19} & 97.34 & \underline{98.35} \\
\bottomrule
\end{tabular}%
}
\end{table}

\subsection{Multi-prompt Generation Results}
We first compare SWIFT with representative existing methods under the multi-prompt setting to evaluate its performance on long-horizon interactive video generation.

\textbf{Text-video alignment and efficiency.}
Table~\ref{tab:multi_prompt_60s} further reports clip-wise CLIP scores and inference throughput. SWIFT achieves the highest CLIP scores in five out of six prompt intervals and remains very close to the best result in the first interval, showing consistently stronger semantic alignment immediately after prompt transitions. The semantic injection cache enables SWIFT to reliably sustain strong prompt alignment throughout successive prompt changes by continuously correcting residual semantics in historical memory. At the same time, SWIFT reaches the highest throughput of 22.6 FPS, outperforming LongLive at 20.7 FPS and substantially exceeding the other baselines. 

\textbf{Video quality.}
Visual quality is comprehensively evaluated using multiple metrics from VBench~\cite{huang2024vbench}, including aesthetic quality, background consistency, imaging quality, motion smoothness, subject consistency, and temporal flickering. These metrics provide a broad assessment of visual fidelity, long range coherence, and temporal stability in generated videos.
As shown in Table~\ref{tab:multi_prompt_quality}, SWIFT delivers strong overall visual quality under the multi-prompt setting. It achieves the best results in aesthetic quality, imaging quality, and motion smoothness, while remaining competitive in background consistency, subject consistency, and temporal flickering. These results indicate that SWIFT improves visual fidelity and temporal stability without compromising long-range consistency. We attribute this improvement to the proposed memory design, which facilitates semantic transitions while retaining useful historical context. The results demonstrate that SWIFT improves prompt-following ability under sequential prompt control while also providing higher inference efficiency.

\begin{table}[t]
\centering
\small
\caption{\textbf{Ablation study under the multi-prompt 60-second setting}. The CLIP score is computed at intervals aligned with the prompt switching.}
\label{tab:ablation_multi_prompt_60s}
\setlength{\tabcolsep}{2pt}
\renewcommand{\arraystretch}{1.05}
\definecolor{lightgrayrow}{gray}{0.93}
\resizebox{\columnwidth}{!}{%
\begin{tabular}{lccccccccc}
\toprule
\multirow{2}{*}{\textbf{Variant}} 
& \multirow{2}{*}{\makecell{Quality\\Score$\uparrow$}}
& \multirow{2}{*}{\makecell{Consistency\\Score$\uparrow$}}
& \multirow{2}{*}{\makecell{Throughput\\(FPS)$\uparrow$}}
& \multicolumn{6}{c}{CLIP Score$\uparrow$} \\
\cmidrule(lr){5-10}
& & & & 0--10\,s & 10--20\,s & 20--30\,s & 30--40\,s & 40--50\,s & 50--60\,s \\
\midrule
No-Memory         & 80.94 & 90.87 & \textbf{25.1} & 24.83 & 22.74 & 21.68 & 21.35 & 21.12 & 20.94 \\
Sink              & 82.19 & 92.52 & \underline{24.5} & 25.58 & 23.40 & 22.45 & 22.67 & 22.59 & 22.82 \\
Sink+SIC          & \underline{84.26} & \underline{95.74} & 20.5 & 26.11 & \underline{25.12} & \underline{23.98} & \underline{23.77} & \underline{23.95} & \underline{23.84} \\
Fixed             & 83.71 & 94.96 & 20.9 & \textbf{26.67} & 24.96 & 23.81 & 23.58 & 23.64 & 23.42 \\
\rowcolor{lightgrayrow}
\textbf{SWIFT} & \textbf{84.97} & \textbf{96.62} & 22.6 & \underline{26.53} & \textbf{25.86} & \textbf{24.71} & \textbf{24.04} & \textbf{24.68} & \textbf{24.13}  \\
\bottomrule
\end{tabular}%
}
\end{table}

\vspace{-6pt}
\setlength{\textfloatsep}{5pt}
\begin{figure*}[t]
  \centering
  \includegraphics[width=\textwidth]{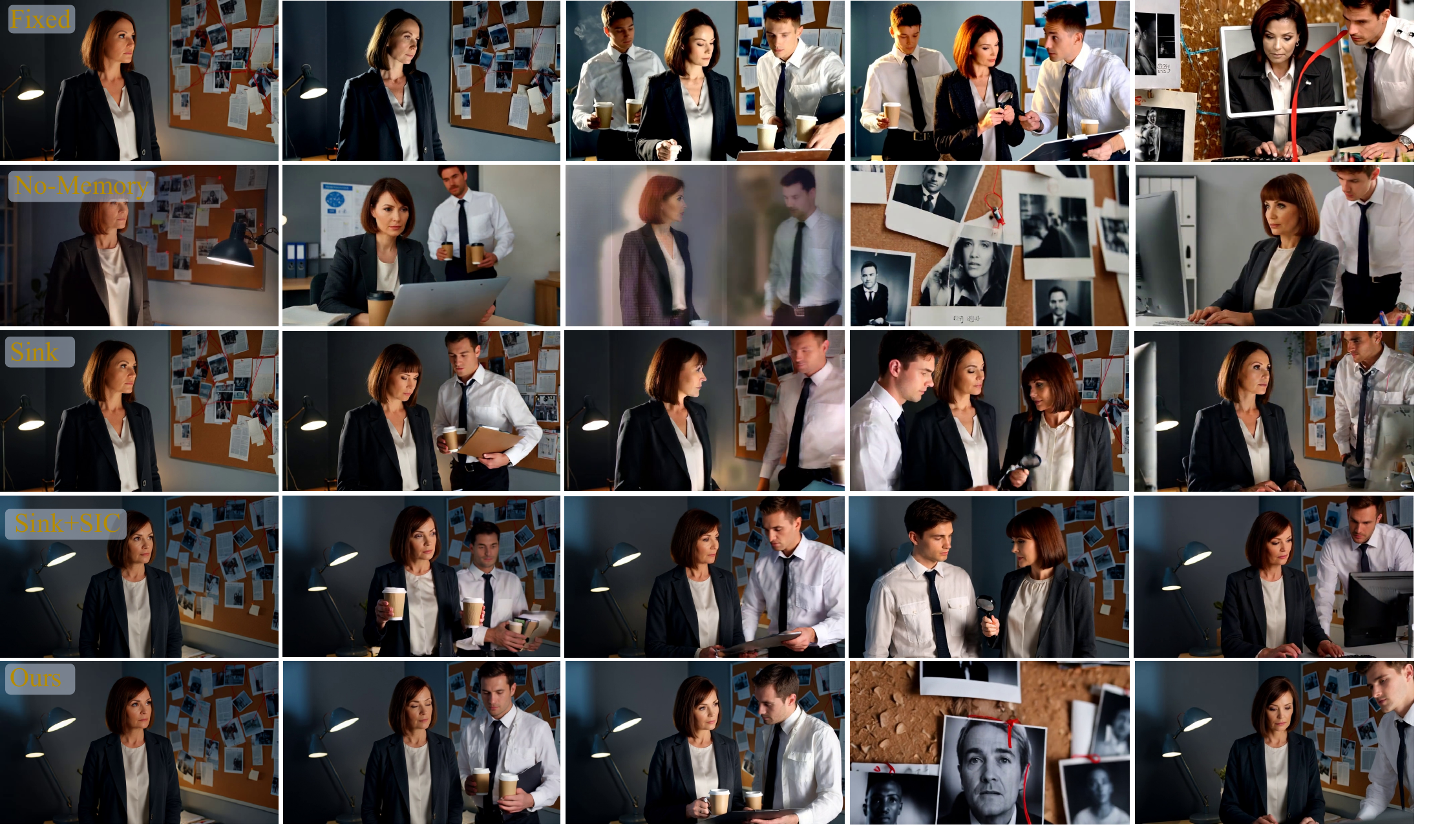}
  \caption{\textbf{Qualitative comparison of different memory variants under the multi-prompt 60-second setting.} \textit{Fixed} uses a constant local window. \textit{No-Memory} removes additional transition-aware memory. \textit{Sink} retains only sink memory. \textit{Sink+SIC} adds the Semantic Injection Cache on top of sink memory. \textit{Ours} denotes the full SWIFT model, which achieves more coherent prompt transitions and better long-range visual consistency.}
  \label{fig:ablation}
\end{figure*}

\vspace{-6pt}
\subsection{Ablation Studies}
\vspace{-4pt}
We conduct ablation studies on the two core components of the framework, namely the Semantic Injection Cache and the Adaptive Dynamic Window, under the 60-second interactive multi-prompt video generation setting with six consecutive prompts.

\textbf{Memory Variants.} Specifically, variants with different memory mechanisms are constructed based on the design in Section~\ref{subsec:semantic_injection_cache}. \textit{No-Memory} removes the proposed memory augmentation and performs generation without additional transition-aware memory. \textit{Sink} retains only the sink memory for long-range context preservation. \textit{Sink+SIC} further introduces the Semantic Injection Cache on top of sink memory to support semantic adaptation across prompt transitions.

\textbf{Window Variants.} Following the design in Section~\ref{subsec:adaptive_dynamic_window}, we further compare different cache capacity strategies. Here, \textit{Fixed} uses a constant static local window with size 12 throughout the entire generation process, without adaptive adjustment across prompt phases.
Empirical results across all variants collectively underscore the necessity of our integrated memory design in effectively mitigating semantic drift while preserving long-range structural integrity.

Table~\ref{tab:ablation_multi_prompt_60s} reports inference throughput together with clip-wise CLIP scores over all six prompt intervals. The full SWIFT achieves the strongest overall semantic alignment performance, which further verifies the effectiveness of combining Semantic Injection Cache with adaptive temporal memory allocation for sustained prompt following under long-horizon multi-prompt generation. High throughput in \textit{No-Memory} and \textit{Sink} is offset by significant degradation in video quality and consistency metrics. Such performance gaps underscore the insufficiency of raw efficiency for sustaining coherent multi-prompt generation over long horizons. Meanwhile, \textit{Sink+SIC} yields performance competitive with SWIFT in visual metrics but exhibits a trade-off through noticeably lower inference throughput. Compared with \textit{Fixed}, SWIFT achieves higher throughput and better overall generation quality, showing the advantage of adaptive temporal memory allocation over a constant window strategy. 

As shown in Figure~\ref{fig:ablation}, across the visualized frame sequences, the full SWIFT produces more coherent semantic transitions and more stable long-range visual continuation than the ablated variants. This highlights the importance of coupling semantic memory refinement with phase-aware temporal allocation in multi-prompt long-video generation.
\begin{figure*}[t]
  \centering
  \includegraphics[width=\textwidth,height=0.6\textheight,keepaspectratio]{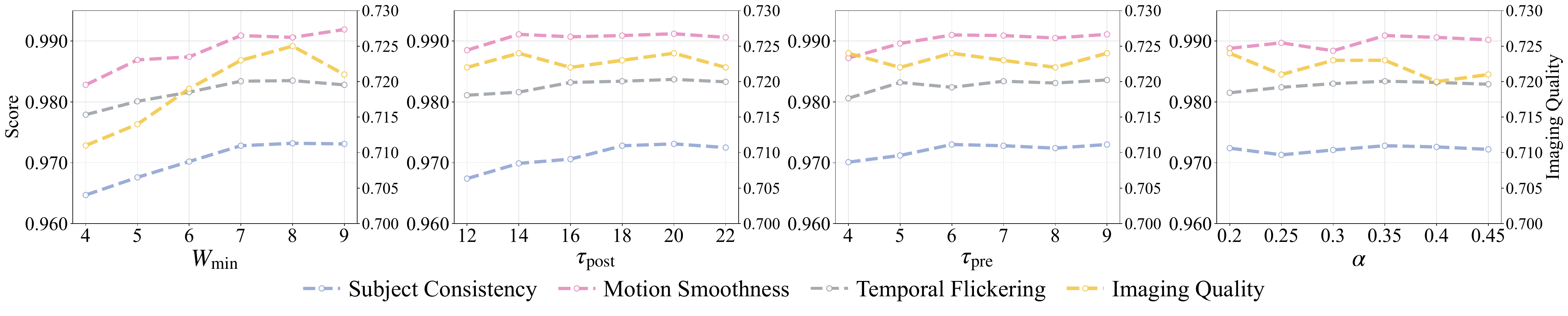}
  \caption{\textbf{Sensitivity analysis of the adaptive dynamic window.} The figure reports the variation of representative visual quality metrics under different hyperparameter settings of the adaptive window, including the minimum window size and the phase scheduling factors.}
  \label{fig:ses}
\end{figure*}

\subsection{Sensitivity Analysis}
Long-horizon video generation requires the model to balance temporal coherence and computational efficiency under constrained attention budgets \cite{li2025longdiff, li2026long}. Sensitivity analyses are conducted on key hyperparameters of both the Semantic Injection Cache and the Adaptive Dynamic Window, described in Section~\ref{subsec:semantic_injection_cache} and Section~\ref{subsec:adaptive_dynamic_window}, respectively. All other hyperparameters remain the same as the default settings. Figure~\ref{fig:ses} illustrates the sensitivity of representative visual quality metrics to adaptive window hyperparameter settings, including the minimum window size and the phase scheduling factors.

\textbf{Minimum window size $W_{\min}$.} The minimum window size $W_{\min}$ determines the lower bound of the temporal read budget in stable prompt phases. Its value affects the trade-off between inference efficiency and local temporal context preservation. A very small $W_{\min}$ leads to clear degradation in generation quality, while an excessively large $W_{\min}$ reduces the efficiency gain of adaptive memory allocation and yields limited additional benefit.

\textbf{Phase scheduling factors $\tau_{\text{post}}$ and $\tau_{\text{pre}}$.} The parameters $\tau_{\text{post}}$ and $\tau_{\text{pre}}$ control how fast the temporal window shrinks after a prompt switch and how early it expands before the next boundary. Across a broad range of $\tau_{\text{post}}$ and $\tau_{\text{pre}}$ values, the visual quality metrics remain largely stable, indicating that SWIFT is robust to moderate changes in phase scheduling. The result indicates that effective multi-prompt generation can be maintained without highly sensitive hyperparameter tuning.

\textbf{Semantic mixing coefficient $\alpha$.} The coefficient $\alpha$ controls the balance between visual memory and semantic identity in the segment-level semantic anchor. Across different $\alpha$ values, the representative quality metrics remain consistently stable. The stable trend across different $\alpha$ values shows that the segment anchor can maintain an effective balance between semantic identity and visual memory over a broad mixing range. These trends motivate the default setting by favoring stable quality under a reduced average memory budget.

\vspace{-6pt}
\subsection{Additional Experiments}
Appendix~\ref{inject_patterns} further analyzes the temporal injection pattern of Semantic Injection Cache by comparing one-shot injection at prompt boundaries, continuously decayed injection, and constant injection throughout generation. Per-block efficiency traces are provided in Appendix~\ref{subsec:block_efficiency_trace} to examine latency, memory read budget, and GPU memory usage across prompt-switch boundaries and stable generation segments. We further discuss multi-prompt video generation under different video lengths in Appendix~\ref{subsec:different_video_lengths} to evaluate the scalability of SWIFT across varying temporal horizons.

\vspace{-4pt}
\section{Conclusion}
In this work, we introduce SWIFT, a training-free memory framework for efficient multi prompt long video generation. SWIFT enables prompt transitions through lightweight semantic injection instead of repeated cache reconstruction. Semantic Injection Cache updates historical video memory with transition-aware prompt signals and preserves motion continuity and scene structure. Adaptive Dynamic Window allocates temporal memory according to prompt phase, using larger context around prompt switches and compact context during stable generation. Segment-level semantic anchors preserve long range prompt conditioned history under reduced local attention. The proposed design avoids expensive full memory rebuilding and provides a more practical path toward scalable interactive long video generation. Extensive experiments show that SWIFT improves prompt responsiveness, temporal coherence, and visual quality under sequential prompt control, and achieves 22.6 FPS on a single H100 GPU in the multi-prompt setting.

\bibliography{neurips_2026}
\bibliographystyle{abbrv}


\clearpage
\appendix
\section{Use of Large Language Models}

Large language models were used only as auxiliary writing tools during the preparation of this manuscript. Their use was limited to improving grammar, readability, and stylistic consistency of author-written text. They were not used to formulate the research idea, design the method, conduct experiments, analyze results, or derive conclusions. All technical content, experimental settings, reported results, and interpretations were produced and checked by the authors.  The authors remain fully responsible for the correctness, integrity, and final content of the manuscript.

\section{Broader Impacts}

SWIFT improves efficient and controllable long-video generation. This capability can support film production, digital content creation, education, and interactive media. The training-free design also reduces extra training cost and lowers the barrier for adapting long-video generators to sequential user prompts.

The same capability may also increase risks from synthetic media. Possible misuse includes misleading video generation, impersonation, and harmful content creation. These risks are shared by many generative video systems. Responsible deployment should include watermarking, provenance tracking, content moderation, and user-facing disclosure. SWIFT is an inference-time method and does not introduce new training data or personal data collection. All text prompts used to generate reasoning were generated by Qwen2-72B-Instruct \cite{yang2025qwen3}, and are harmless, safe, and intended for academic research purposes only.

\section{Detailed Experimental Setup}
\label{deatailed_setup}

\subsection{Compute Resources}
SWIFT is an inference-time framework and requires no additional training for deployment on autoregressive diffusion video generation models. All experiments in this paper are conducted on a single H100 GPU with 80GB memory for inference, which does not impose a high computational requirement in the context of long-video generation.

\subsection{SWIFT Settings}

\textbf{Autoregressive generation.}
All experiments follow the autoregressive inference protocol of LongLive with
Wan2.1-T2V-1.3B as the backbone. Each video is generated block by block with
\(B=3\) frames per block and \(T=240\) output frames in total. The multi-prompt
schedule contains six prompt segments, with switching boundaries at
\(\mathcal{S}=\{40,80,120,160,200\}\). 

\textbf{Diffusion inference.}
The denoising process uses four inference steps,
\(\mathcal{D}=\{1000,750,500,250\}\), with warped denoising enabled. The
timestep shift of the backbone diffusion model is fixed to
\(\lambda_{\mathrm{ts}}=5.0\) for all experiments.

\textbf{Memory configuration.}
The persistent sink size is set to \(N_{\mathrm{sink}}=3\), and the maximum
local attention window is set to \(W_{\max}=12\), corresponding to the
\texttt{local\_attn\_size} used in the implementation. For the adaptive dynamic
window, we use \(W_{\min}=7\), \(\tau_{\mathrm{post}}=18\), and
\(\tau_{\mathrm{pre}}=9\), while the fixed-window ablation uses \(W=12\)
throughout generation. Segment-level semantic anchors are enabled by default:
the visual summary for each completed segment is computed from the most recent
\(R_{\mathrm{anchor}}=6\) frames, and at most \(M_{\mathrm{anchor}}=4\)
historical anchors are retained. The semantic mixing coefficient is
set to \(\alpha=0.35\), and the anchor injection scale is set to
\(\gamma_{\mathrm{anchor}}=0.8\). At each switching boundary
\(s_m \in \mathcal{S}\), the Semantic Injection Cache constructs transient
bridge tokens from \(\Delta p^{(m)}_{\perp}\), which are inserted into memory
and decayed during subsequent generation steps. The semantic bridge decay factor is set to \(\lambda=0.85\)


\subsection{Baselines}
We compare SWIFT with several state-of-the-art inference-time methods for interactive long-video generation. For a fair comparison, all baseline hyperparameters are kept consistent with the default configurations reported in their original papers.

\textbf{MemFlow}~\cite{ji2025memflow} is a dynamic memory-retrieval framework for efficient long-context video generation. It retrieves prompt-relevant historical frames to update the memory bank before each generation chunk, and activates only the most relevant memory tokens during attention, thereby enabling consistent long-video generation with limited computational overhead.

\textbf{FramePack} \cite{zhang2025packing} is a context-packing framework designed for efficient long video generation that optimizes computational throughput and memory usage by packing multiple video frames into a unified context representation.

\textbf{CausVid}~\cite{yin2025slow} is a causal autoregressive video diffusion framework for fast streaming generation. It converts a pretrained bidirectional diffusion transformer into an autoregressive generator and applies video distribution matching distillation to obtain a few-step causal student.

\textbf{Self-Forcing} \cite{huang2025self} is an autoregressive video diffusion training paradigm that mitigates the exposure bias (train-test gap) by training the model on its own generated outputs using key-value (KV) caching and holistic video-level supervision.

\textbf{LongLive} \cite{yang2026longlive} is a real-time interactive video generation framework that introduces a \textit{KV-recache} mechanism to recalculate cached states during prompt transitions, ensuring semantic adherence to new instructions while preserving temporal motion continuity.

\begin{figure*}[t]
\vspace{-6pt}
\centering

\begin{minipage}[t]{0.485\textwidth}
\begin{algorithm}[H]
\caption{Semantic Injection Cache}
\label{alg:semantic_injection_cache}
\Input{Previous prompt signature $p^{(m-1)}$, current prompt signature $p^{(m)}$; value cache of the current layer; recent summary $r$; sink summary $s$}
\Output{Transient semantic bridge entries $B_s$ and $B_r$}
$\Delta p^{(m)} \leftarrow$ Eq.~\eqref{eq:prompt_delta_method}\;
$\rho^{(m)} \leftarrow$ Eq.~\eqref{eq:switch_strength_method}\;
\For{each layer}{
    Estimate local motion direction $m$ from value cache\;
    $\Delta p_{\perp}^{(m)} \leftarrow$ Eq.~\eqref{eq:motion_orthogonal_projection_method}\;
    Extract recent summary $r$ and sink summary $s$\;
    \For{each head}{
        $(g_r,g_s) \leftarrow$ Eq.~\eqref{eq:headwise_gates_method}\;
        $(B_r,B_s) \leftarrow$ Eq.~\eqref{eq:bridge_construction_method}\;
    }
    Append $B_s$ and $B_r$ as transient semantic entries\;
}
Read from sink memory, segment anchors, bridge entries, and local window\;
\KwRet $B_s, B_r$\;
\end{algorithm}
\end{minipage}
\hfill
\begin{minipage}[t]{0.485\textwidth}
\begin{algorithm}[H]
\caption{Adaptive Dynamic Window}
\label{alg:adaptive_dynamic_window}
\Input{Current block position $t$; active segment start $s_m$; next switch boundary $s_{m+1}$; window limits $W_{\min},W_{\max}$; recent value summary $u^{(m)}$; prompt signature $p^{(m)}$}
\Output{Adaptive window size $W_t$ and segment anchor $A^{(m)}$}
Determine active prompt segment for $t$\;
$a_t \leftarrow \max(0,t-s_m)$\;
Compute distance $d_t$ to next switch boundary $s_{m+1}$\;
Compute $w_{\text{post}}(t)$ and $w_{\text{pre}}(t)$\;
$w_t \leftarrow$ Eq.~\eqref{eq:phase_weight_method}\;
$W_t \leftarrow$ Eq.~\eqref{eq:adaptive_window_method}\;
Select effective local window according to $W_t$\;
\If{current prompt segment ends}{
    Summarize recent value cache as $u^{(m)}$\;
    $A^{(m)} \leftarrow$ Eq.~\eqref{eq:segment_anchor_method}\;
    Append $A^{(m)}$ into structured memory\;
}
Read from sink memory, bridge entries, anchors, and adaptive local window\;
\KwRet $W_t, A^{(m)}$\;
\end{algorithm}
\end{minipage}

\vspace{-6pt}
\end{figure*}

\section{Proof of Theorem~\ref{thm:motion_neutral_projection}}
\label{sec:appendix_theoretical_justification}

This appendix provides the full proof of
Theorem~\ref{thm:motion_neutral_projection}. We first formalize the local
cache response model used to define motion-neutral semantic injection, and then
derive the closed-form projection as the locally optimal prompt-preserving
update.

\subsection{Setup and Assumptions}

Let $\mathcal{H}$ be the Hilbert space of value-cache representations, equipped
with inner product $\langle\cdot,\cdot\rangle$ and induced norm $\|\cdot\|_2$.
During autoregressive generation, the cached video representation evolves as
\begin{equation}
    v_1,v_2,\ldots,v_t\in\mathcal{H}.
\end{equation}
At a prompt switch from segment $m-1$ to segment $m$, the projected prompt
displacement is
\begin{equation}
    \Delta p^{(m)} = p^{(m)}-p^{(m-1)} \in \mathcal{H}.
\end{equation}
The local cache tangent is estimated from consecutive value-cache summaries:
\begin{equation}
    m = v_t-v_{t-1}, \qquad m\neq 0 .
\end{equation}

\paragraph{Local cache tangent.}
We assume the cache trajectory is locally first-order smooth around the current
state. For a small temporal step $\Delta t$, this gives
\begin{equation}
    v_t
    =
    v_{t-1}
    +
    \dot v_{t-1}\Delta t
    +
    O(\Delta t^2).
\end{equation}
Thus, $m=v_t-v_{t-1}$ provides a finite-difference estimate of the dominant
local tangent direction of short-term cache dynamics. We denote the corresponding
motion tangent subspace by
\begin{equation}
    \mathcal{T}_t = \operatorname{span}\{m\},
\end{equation}
and its orthogonal complement by
\begin{equation}
    \mathcal{N}_t = \mathcal{T}_t^{\perp}
    =
    \{z\in\mathcal{H}:\langle z,m\rangle=0\}.
\end{equation}

\paragraph{First-order motion response.}
For a small semantic injection $v_t\mapsto v_t+\eta\Delta x$, its tangential
effect on the local cache trajectory is measured by the projection onto
$\mathcal{T}_t$. The orthogonal projection onto $\mathcal{T}_t$ is
\begin{equation}
    \Pi_{\mathcal{T}_t}(z)
    =
    \frac{\langle z,m\rangle}{\|m\|_2^2}m .
\end{equation}
Therefore, the first-order tangential variation caused by the injection is
\begin{equation}
\begin{aligned}
    \Pi_{\mathcal{T}_t}(v_t+\eta\Delta x)
    -
    \Pi_{\mathcal{T}_t}(v_t)
    &=
    \eta \Pi_{\mathcal{T}_t}(\Delta x)  \\
    &=
    \eta
    \frac{\langle \Delta x,m\rangle}{\|m\|_2^2}m .
\end{aligned}
\end{equation}
Its magnitude is proportional to $|\langle \Delta x,m\rangle|$. Hence, an update
does not perturb the local motion tangent to first order if and only if
\begin{equation}
    \Pi_{\mathcal{T}_t}(\Delta x)=0
    \quad\Longleftrightarrow\quad
    \langle \Delta x,m\rangle=0 .
\end{equation}
This defines the admissible motion-neutral subspace $\mathcal{N}_t$.

\paragraph{Prompt-preserving semantic injection.}
The raw prompt displacement $\Delta p^{(m)}$ represents the desired semantic
change at the prompt boundary. Since only updates in $\mathcal{N}_t$ are
motion-neutral to first order, the injected semantic update should retain as
much of $\Delta p^{(m)}$ as possible while remaining in $\mathcal{N}_t$. This
leads to the minimum-distortion problem
\begin{equation}
\label{eq:appendix_motion_neutral_problem}
    \Delta p_{\perp}^{(m)}
    =
    \arg\min_{\Delta x\in\mathcal{H}}
    \|\Delta x-\Delta p^{(m)}\|_2^2
    \quad
    \mathrm{s.t.}\quad
    \langle \Delta x,m\rangle=0 .
\end{equation}
Equivalently, the solution is the motion-neutral component of
$\Delta p^{(m)}$ that maximally preserves the prompt displacement under the
first-order motion-neutrality constraint.

\subsection{Proof}

\begin{proof}[Proof of Theorem~\ref{thm:motion_neutral_projection}]
For compactness, denote
\begin{equation}
    q := \Delta p^{(m)} .
\end{equation}
The feasible set of Eq.~\eqref{eq:appendix_motion_neutral_problem} is
\begin{equation}
    \mathcal{C}
    :=
    \{z\in\mathcal{H}:\langle z,m\rangle=0\}
    =
    \mathcal{N}_t
    =
    m^\perp .
\end{equation}
Since $m\neq 0$, $\mathcal{C}$ is a closed linear subspace of the Hilbert space
$\mathcal{H}$. By the Hilbert projection theorem, every $q\in\mathcal{H}$ admits
a unique closest point in $\mathcal{C}$. Hence the minimizer of
Eq.~\eqref{eq:appendix_motion_neutral_problem} exists and is unique, and is the
orthogonal projection of $q$ onto $\mathcal{C}$.

It remains to compute this projection. Since
\begin{equation}
    \mathcal{C}^{\perp}
    =
    \operatorname{span}\{m\},
\end{equation}
there exists a unique orthogonal decomposition
\begin{equation}
\label{eq:appendix_decomposition}
    q = z_{\perp}+\alpha m,
    \qquad
    z_{\perp}\in\mathcal{C},
    \quad
    \alpha m\in\mathcal{C}^{\perp}.
\end{equation}
Taking the inner product of Eq.~\eqref{eq:appendix_decomposition} with $m$ gives
\begin{equation}
    \langle q,m\rangle
    =
    \langle z_{\perp},m\rangle
    +
    \alpha\|m\|_2^2 .
\end{equation}
Because $z_{\perp}\in\mathcal{C}$, we have $\langle z_{\perp},m\rangle=0$.
Therefore,
\begin{equation}
    \alpha
    =
    \frac{\langle q,m\rangle}{\|m\|_2^2}.
\end{equation}
Substituting this value back into Eq.~\eqref{eq:appendix_decomposition} yields
\begin{equation}
    z_{\perp}
    =
    q
    -
    \frac{\langle q,m\rangle}{\|m\|_2^2}m .
\end{equation}
Returning to $q=\Delta p^{(m)}$, we obtain
\begin{equation}
\label{eq:appendix_exact_projection}
    \Delta p_{\perp}^{(m)}
    =
    \Delta p^{(m)}
    -
    \frac{\langle \Delta p^{(m)},m\rangle}{\|m\|_2^2}m .
\end{equation}

We next verify the minimum-distortion optimality. For any feasible
$\Delta x\in\mathcal{C}$, since $q-z_{\perp}\in\mathcal{C}^{\perp}$ and
$\Delta x-z_{\perp}\in\mathcal{C}$, the two terms are orthogonal:
\begin{equation}
    \langle \Delta x-z_{\perp}, q-z_{\perp}\rangle = 0 .
\end{equation}
Hence the Pythagorean identity gives
\begin{equation}
\begin{aligned}
    \|\Delta x-q\|_2^2
    &=
    \|(\Delta x-z_{\perp})-(q-z_{\perp})\|_2^2 \\
    &=
    \|\Delta x-z_{\perp}\|_2^2
    +
    \|q-z_{\perp}\|_2^2        \\
    &\geq
    \|q-z_{\perp}\|_2^2 .
\end{aligned}
\end{equation}
Equality holds if and only if $\Delta x=z_{\perp}$. Therefore,
$z_{\perp}$ is the unique minimum-distortion update satisfying
$\langle \Delta x,m\rangle=0$.

Finally, we show that the same solution also maximally preserves the prompt
transition within the motion-neutral subspace. For any $\Delta x\in\mathcal{C}$,
using $q=z_{\perp}+\alpha m$ and $\langle \Delta x,m\rangle=0$, we have
\begin{equation}
    \langle \Delta x,q\rangle
    =
    \langle \Delta x,z_{\perp}\rangle .
\end{equation}
Thus, by the Cauchy--Schwarz inequality, any motion-neutral update with no larger
energy than $z_{\perp}$ satisfies
\begin{equation}
\label{eq:appendix_semantic_response_bound}
    \langle \Delta x,q\rangle
    =
    \langle \Delta x,z_{\perp}\rangle
    \leq
    \|\Delta x\|_2\|z_{\perp}\|_2
    \leq
    \|z_{\perp}\|_2^2
    =
    \langle z_{\perp},q\rangle .
\end{equation}
The upper bound is achieved by $\Delta x=z_{\perp}$. Hence
$z_{\perp}=\Delta p_{\perp}^{(m)}$ is not only the closest admissible vector to
the raw prompt displacement, but also the locally optimal motion-neutral update
that preserves the largest first-order prompt component under the same energy
budget. This proves Eq.~\eqref{eq:appendix_exact_projection} and completes the
proof.
\end{proof}

\subsection{Stabilized Implementation}

The exact projection in Eq.~\eqref{eq:appendix_exact_projection} requires
division by $\|m\|_2^2$. When the finite-difference motion estimate has very
small magnitude, this denominator may become numerically unstable. We therefore
use the stabilized implementation
\begin{equation}
\label{eq:appendix_stabilized_projection}
    \widehat{\Delta p}_{\perp}^{(m)}
    =
    \Delta p^{(m)}
    -
    \frac{\langle \Delta p^{(m)},m\rangle}
    {\|m\|_2^2+\epsilon}m,
    \qquad
    \epsilon>0 .
\end{equation}
This update converges to the exact projection as $\epsilon\rightarrow 0$:
\begin{equation}
    \lim_{\epsilon\rightarrow 0}
    \widehat{\Delta p}_{\perp}^{(m)}
    =
    \Delta p_{\perp}^{(m)} .
\end{equation}
Its residual first-order motion response is
\begin{equation}
\begin{aligned}
    \left\langle
    \widehat{\Delta p}_{\perp}^{(m)},m
    \right\rangle
    &=
    \left\langle
    \Delta p^{(m)}
    -
    \frac{\langle \Delta p^{(m)},m\rangle}
    {\|m\|_2^2+\epsilon}m,
    m
    \right\rangle        \\
    &=
    \frac{\epsilon}{\|m\|_2^2+\epsilon}
    \langle \Delta p^{(m)},m\rangle .
\end{aligned}
\end{equation}
Thus, the stabilized update preserves the same projection principle, approaches
exact motion orthogonality as $\epsilon$ decreases, and avoids numerical
amplification when the estimated local motion direction is unreliable.

\section{Further Experiments}
\label{sec:further_experiments}
This section provides supplementary ablation studies and detailed analyses regarding the architectural components of SWIFT. These evaluations offer further insights into the system's performance across various interactive generation scenarios.

\subsection{Temporal Injection Patterns}
\label{inject_patterns}
We investigate the influence of different injection schedules for the Semantic Injection Cache detailed in Section \ref{subsec:semantic_injection_cache} by comparing one-shot, continuously decayed, and constant injection modes. Specifically, one-shot injection inserts the full semantic update at the switching boundary without decay, continuously decayed injection denotes the default semantic injection strategy, and constant injection progressively inserts a fixed proportion of the semantic update over time. This analysis quantifies the trade-off between immediate semantic adaptation and the preservation of long-term temporal stability. Experiments are conducted in the multi-prompt setting by generating 100 videos of 60 seconds, and the results are evaluated using VBench \cite{huang2024vbench}.

Table~\ref{tab:semantic_injection_schedule} shows the quantitative results in terms of aesthetic quality, background consistency, imaging quality, motion smoothness, subject consistency, and temporal flickering. Compared with one-shot injection and constant injection, the continuously decayed injection strategy achieves the best overall visual quality, as it provides sufficient semantic guidance immediately after each prompt transition while gradually reducing the strength of injected bridge tokens in subsequent steps. This design prevents abrupt semantic overwriting caused by constant injection and avoids the insufficient long-range adaptation of one-shot injection, thereby yielding more coherent subject appearance, smoother motion, and reduced temporal flickering across prompt switches.

Table~\ref{yuyi_clip} reports the segment-wise CLIP scores under different temporal injection schedules. The results show that continuously decayed injection maintains stronger prompt alignment after semantic transitions, indicating that gradual decay enables more effective semantic adaptation than one-shot or constant injection. Figure~\ref{fig:yuyi} presents representative generated frames for the one-shot, constant, and continuously decayed injection schedules under the same multi-prompt sequence.
\begin{table}[h]
\centering
\caption{\textbf{Quantitative comparison under multi-prompt 60-second setting} with different injection schedules for the Semantic Injection Cache detailed in Section~\ref{subsec:semantic_injection_cache}, including one-shot, continuously decayed, and constant injection modes.}
\label{tab:semantic_injection_schedule}
\definecolor{lightgrayrow}{gray}{0.93}
\setlength{\tabcolsep}{2pt}
\renewcommand{\arraystretch}{1.05}
\resizebox{\columnwidth}{!}{%
\begin{tabular}{lcccccc}
\toprule
\textbf{Injection Schedule}
& \makecell{Aesthetic\\Quality$\uparrow$}
& \makecell{Background\\Consistency$\uparrow$}
& \makecell{Imaging\\Quality$\uparrow$}
& \makecell{Motion\\Smoothness$\uparrow$}
& \makecell{Subject\\Consistency$\uparrow$}
& \makecell{Temporal\\Flickering$\uparrow$} \\
\midrule
One-shot Injection              & 59.96 & 96.01 & 71.61 & 99.06 & 97.03 & 98.28 \\
Constant Injection              & 60.03 & 95.95 & 71.59 & 99.09 & 96.91 & 98.31 \\
Continuously Decayed Injection & \textbf{60.23} & \textbf{96.28} & \textbf{72.32} & \textbf{99.19} & \textbf{97.34} & \textbf{98.35} \\
\bottomrule
\end{tabular}%
}
\end{table}

\begin{table}[t]
\centering
\caption{\textbf{Quantitative comparison of temporal injection schedules.} We evaluate one-shot, constant, and continuously decayed injection under the multi-prompt 60-second setting. The CLIP \cite{radford2021learning} score is computed for each 10-second segment according to the corresponding prompt semantics, where continuously decayed injection achieves the best overall alignment across prompt transitions.}
\label{yuyi_clip}
\definecolor{lightgrayrow}{gray}{0.93}
\setlength{\tabcolsep}{2.5pt}
\renewcommand{\arraystretch}{1.0}
\small
\resizebox{0.8\columnwidth}{!}{%
\begin{tabular}{lcccccc}
\toprule
\multirow{2}{*}{\textbf{Injection Schedule}} 
& \multicolumn{6}{c}{CLIP Score$\uparrow$} \\
\cmidrule(lr){2-7}
& 0--10\,s & 10--20\,s & 20--30\,s & 30--40\,s & 40--50\,s & 50--60\,s \\
\midrule
One-shot Injection 
& \textbf{26.60} & 25.42 & 23.74 & 23.85 & 24.18 & 23.67 \\
Constant Injection
& 26.32 & 25.77 & 23.18 & 23.76 & 24.25 & 23.89 \\
Continuously Decayed Injection 
& {26.53} & \textbf{25.86} & \textbf{24.71} & \textbf{24.04} & \textbf{24.68} & \textbf{24.13} \\
\bottomrule
\end{tabular}%
}
\end{table}

\setlength{\textfloatsep}{5pt}
\begin{figure*}[t]
  \centering
  \includegraphics[width=\textwidth]{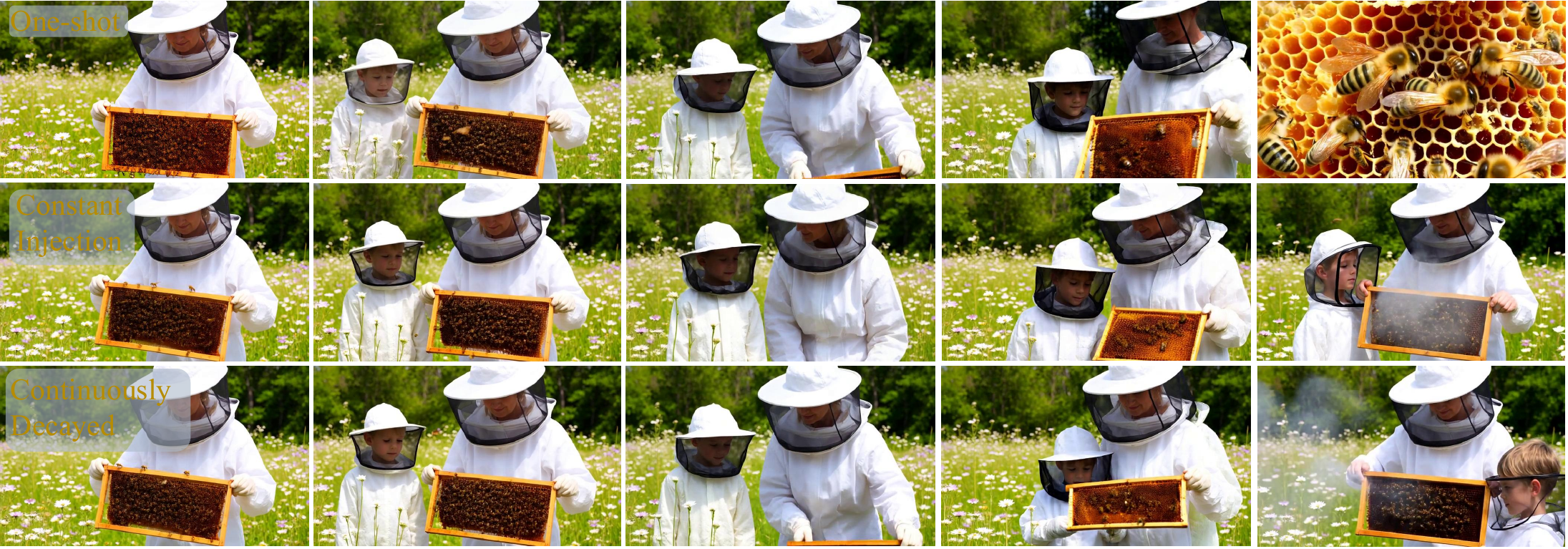}
  \caption{\textbf{Qualitative examples of temporal injection schedules.} We visualize generated frames from one-shot, constant, and continuously decayed injection under the same multi-prompt sequence. One-shot injection causes abrupt semantic changes, constant injection shows weaker adaptation after prompt switches, and continuously decayed injection provides smoother semantic transitions while preserving temporal consistency.}
  \label{fig:yuyi}
\end{figure*}

\begin{figure*}[t]
  \centering
  \includegraphics[width=\textwidth]{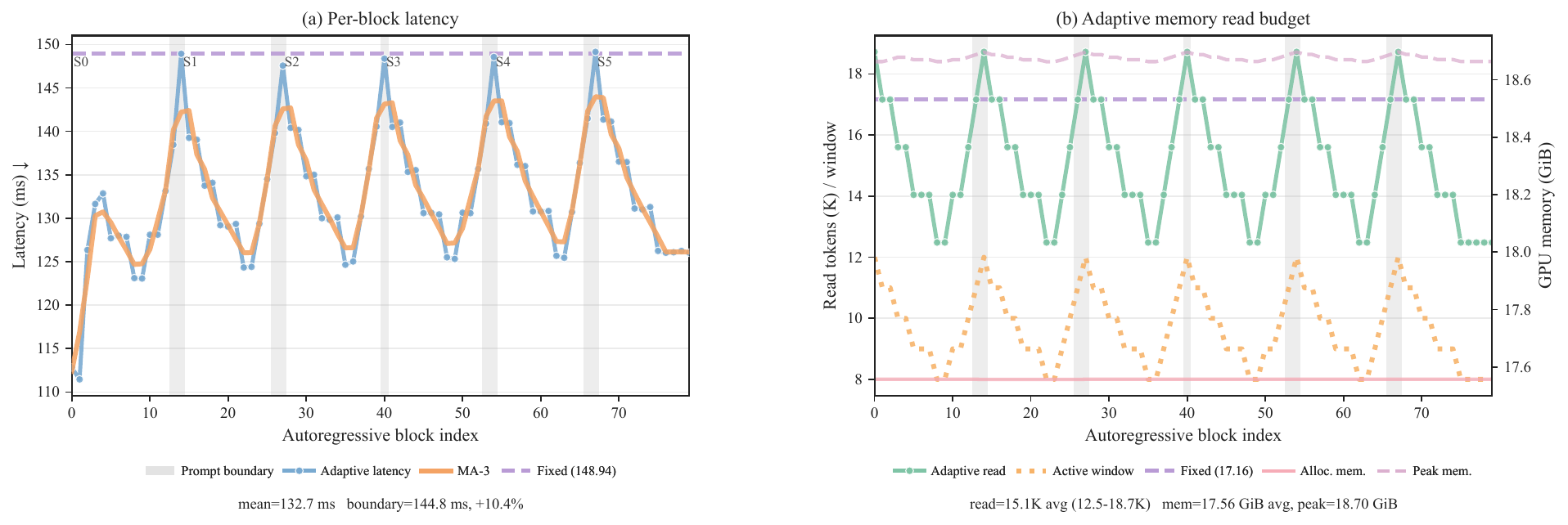}
  \caption{
\textbf{Per-block efficiency of SWIFT.}
Gray bands mark prompt-switch boundaries, and dashed lines denote fixed-window references.
SWIFT shows only mild latency increases at prompt transitions, while its adaptive memory schedule expands the read budget near boundaries and contracts it in stable segments, with nearly constant GPU memory usage.
}
  \label{fig:efficiency}
\end{figure*}

\subsection{Per Block Efficiency Trace}
\label{subsec:block_efficiency_trace}

We further analyze the runtime behavior of SWIFT at the autoregressive block level.
The profiler records block latency, effective memory read budget, active window size, allocated GPU memory, and peak allocated GPU memory.
Prompt switch boundaries are also logged.
These metrics measure the local cost of semantic transition and the memory budget used during stable rollout. We use an autoregressive generation baseline with a fixed window size and ReCache at prompt boundaries as the comparison setting.

Figure~\ref{fig:efficiency} shows the efficiency trace per block.
SWIFT shows mild latency increases near prompt switch boundaries.
The average latency is 132.7 ms, and the average boundary latency is 144.8 ms. The result indicates stable prompt switching without severe runtime spikes.
The fixed window reference uses a constant large local window across all blocks.
This strategy keeps a high read budget during both transition phases and stable phases.
In contrast, SWIFT increases the read budget near prompt boundaries and reduces it during stable segments.
The effective read budget averages 15.1K tokens and ranges from 12.5K to 18.7K tokens.
GPU memory remains nearly constant, with 17.56 GiB average usage and 18.70 GiB peak usage.

The efficiency advantage of SWIFT comes from prompt phase aware memory allocation.
A larger window supplies sufficient context for semantic handover near prompt switches.
A smaller window reduces redundant attention computation during stable generation.
Segment anchors preserve compact long range semantics under the reduced local window.
This design improves runtime efficiency over the fixed window setting while retaining stable memory usage.
\begin{table}[t]
\centering
\caption{\textbf{Quantitative comparison under different multi-prompt video lengths.}
We compare SWIFT with LongLive across 30, 45, 60, 75, and 90 seconds using VBench quality and consistency scores.}
\label{tab:different_video_lengths}
\definecolor{lightgrayrow}{gray}{0.93}
\setlength{\tabcolsep}{3.5pt}
\renewcommand{\arraystretch}{1.08}
\small
\resizebox{0.95\columnwidth}{!}{%
\begin{tabular}{lcccccccccc}
\toprule
\multirow{2}{*}{\textbf{Method}} 
& \multicolumn{2}{c}{\textbf{30s}} 
& \multicolumn{2}{c}{\textbf{45s}} 
& \multicolumn{2}{c}{\textbf{60s}} 
& \multicolumn{2}{c}{\textbf{75s}} 
& \multicolumn{2}{c}{\textbf{90s}} \\
\cmidrule(lr){2-3}
\cmidrule(lr){4-5}
\cmidrule(lr){6-7}
\cmidrule(lr){8-9}
\cmidrule(lr){10-11}
& \makecell{Quality\\Score$\uparrow$}
& \makecell{Consistency\\Score$\uparrow$}
& \makecell{Quality\\Score$\uparrow$}
& \makecell{Consistency\\Score$\uparrow$}
& \makecell{Quality\\Score$\uparrow$}
& \makecell{Consistency\\Score$\uparrow$}
& \makecell{Quality\\Score$\uparrow$}
& \makecell{Consistency\\Score$\uparrow$}
& \makecell{Quality\\Score$\uparrow$}
& \makecell{Consistency\\Score$\uparrow$} \\
\midrule
LongLive~\cite{yang2026longlive}
& 83.64 & 96.08
& 83.76 & 96.31
& 83.88 & 96.48
& 84.47 & 96.56
& {84.58} & 96.67 \\
\rowcolor{lightgrayrow}
\textbf{SWIFT}
& \textbf{84.72} & \textbf{96.43}
& \textbf{84.83} & \textbf{96.51}
& \textbf{84.97} & \textbf{96.62}
& \textbf{85.01} & \textbf{96.74}
& \textbf{85.06} & \textbf{96.89} \\
\bottomrule
\end{tabular}%
}
\end{table}
\subsection{Multi-Prompt Generation with Different Video Lengths}
\label{subsec:different_video_lengths}

We further evaluate SWIFT under different generation lengths to examine its scalability in multi-prompt long-video generation. Specifically, we vary the target video duration across 30, 45, 60, 75, and 90 seconds while keeping the six-prompt interactive setting. The prompt switching boundaries are uniformly scaled according to the target duration, and all remaining model configurations, inference hyperparameters, and evaluation protocols are kept unchanged.

Table~\ref{tab:different_video_lengths} reports the video quality metrics on the VBench benchmark across different generation lengths. Compared with LongLive, SWIFT consistently maintains better overall visual quality across varying temporal horizons, demonstrating the scalability and robustness of the proposed memory mechanism for multi-prompt long-video generation. Figure~\ref{fig:diff_time_case} shows representative frames from the 30-second six-prompt setting, where SWIFT preserves the main subject and interaction structure more consistently across prompt transitions. Under the shorter 30-second generation setting, rapid semantic injection enables SWIFT to respond more effectively to dense prompt switches while maintaining coherent subject appearance and scene continuity.

\begin{figure*}[h]
  \centering
  \includegraphics[width=\textwidth]{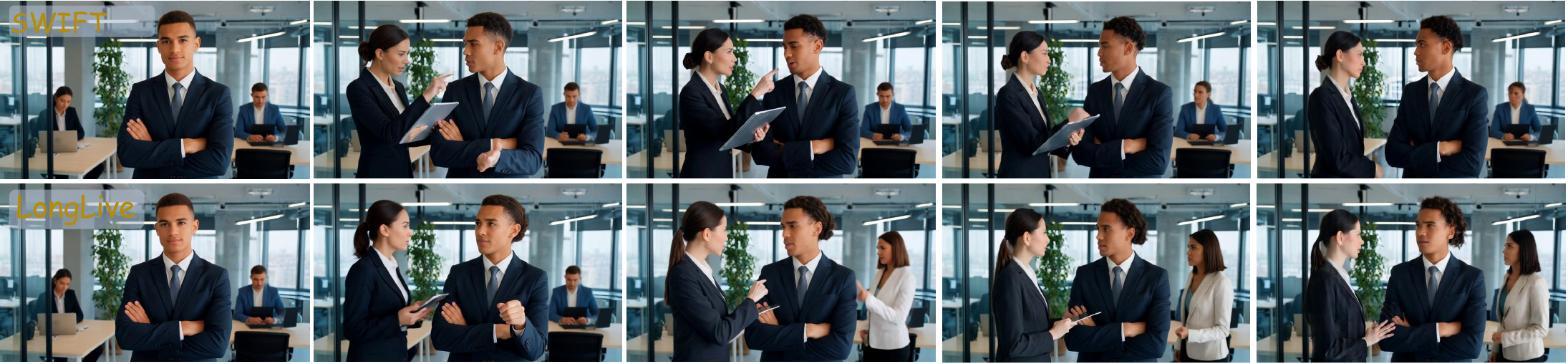}
  \caption{\textbf{Qualitative example of 30-second multi-prompt video generation.} We visualize representative frames generated by SWIFT and LongLive under the same six-prompt sequence.}
  \label{fig:diff_time_case}
\end{figure*}

\subsection{Case Study}
\label{subsec:case_study}

We present additional qualitative results in Figure~\ref{fig:case_study_1} and Figure~\ref{fig:case_study_2} to illustrate the videos generated by SWIFT from multi-prompt inputs. These examples show that SWIFT can follow sequential prompt changes while preserving coherent visual progression across long video generation.

\begin{figure*}[p]
  \centering

  \includegraphics[width=\textwidth,height=0.42\textheight,keepaspectratio]{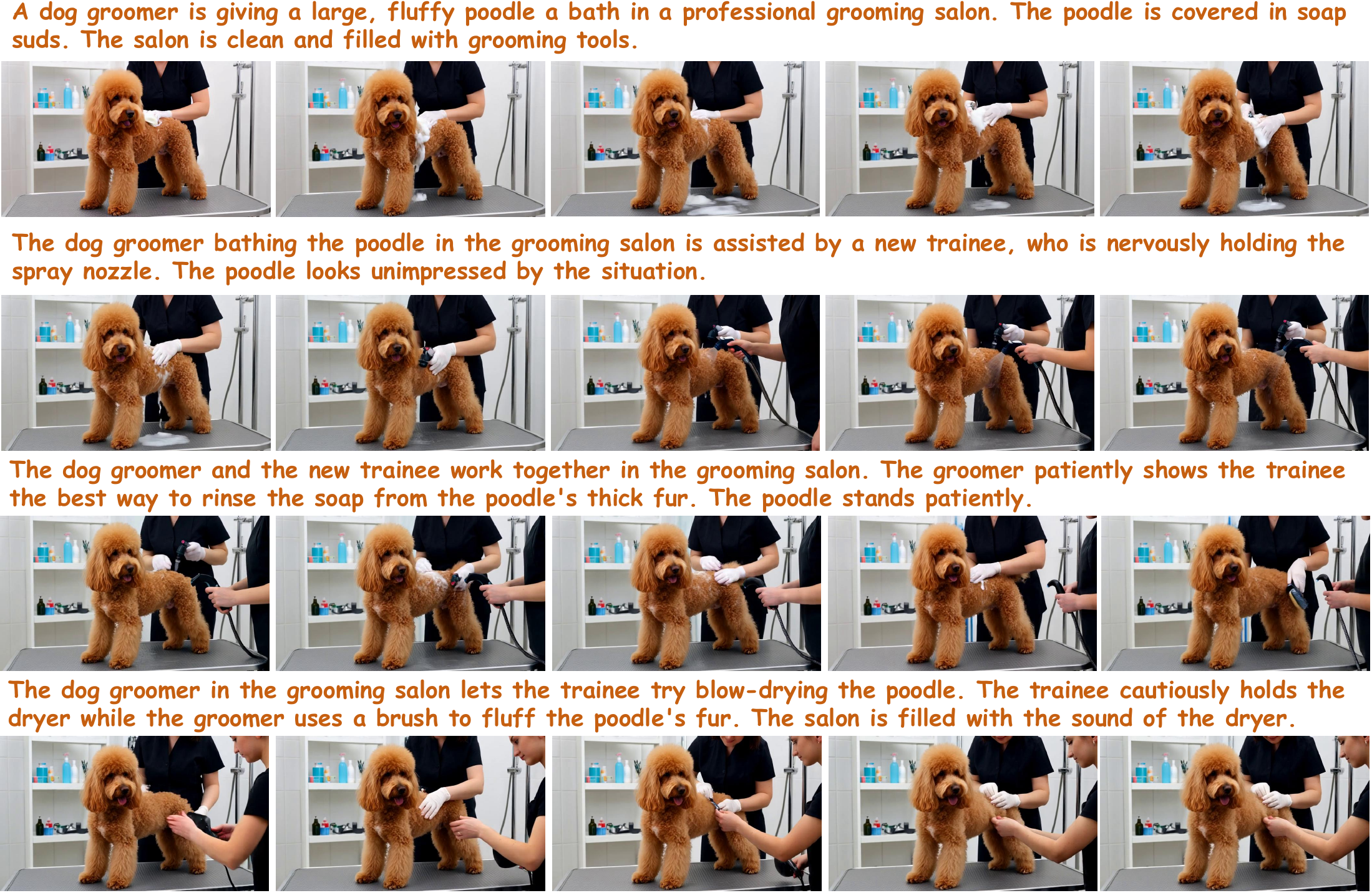}
  \caption{\textbf{Qualitative example of 60-second multi-prompt video generation.} We visualize representative frames generated by SWIFT under a six-prompt sequence.}
  \label{fig:case_study_1}

  \vspace{0.5em}

  \includegraphics[width=\textwidth,height=0.50\textheight,keepaspectratio]{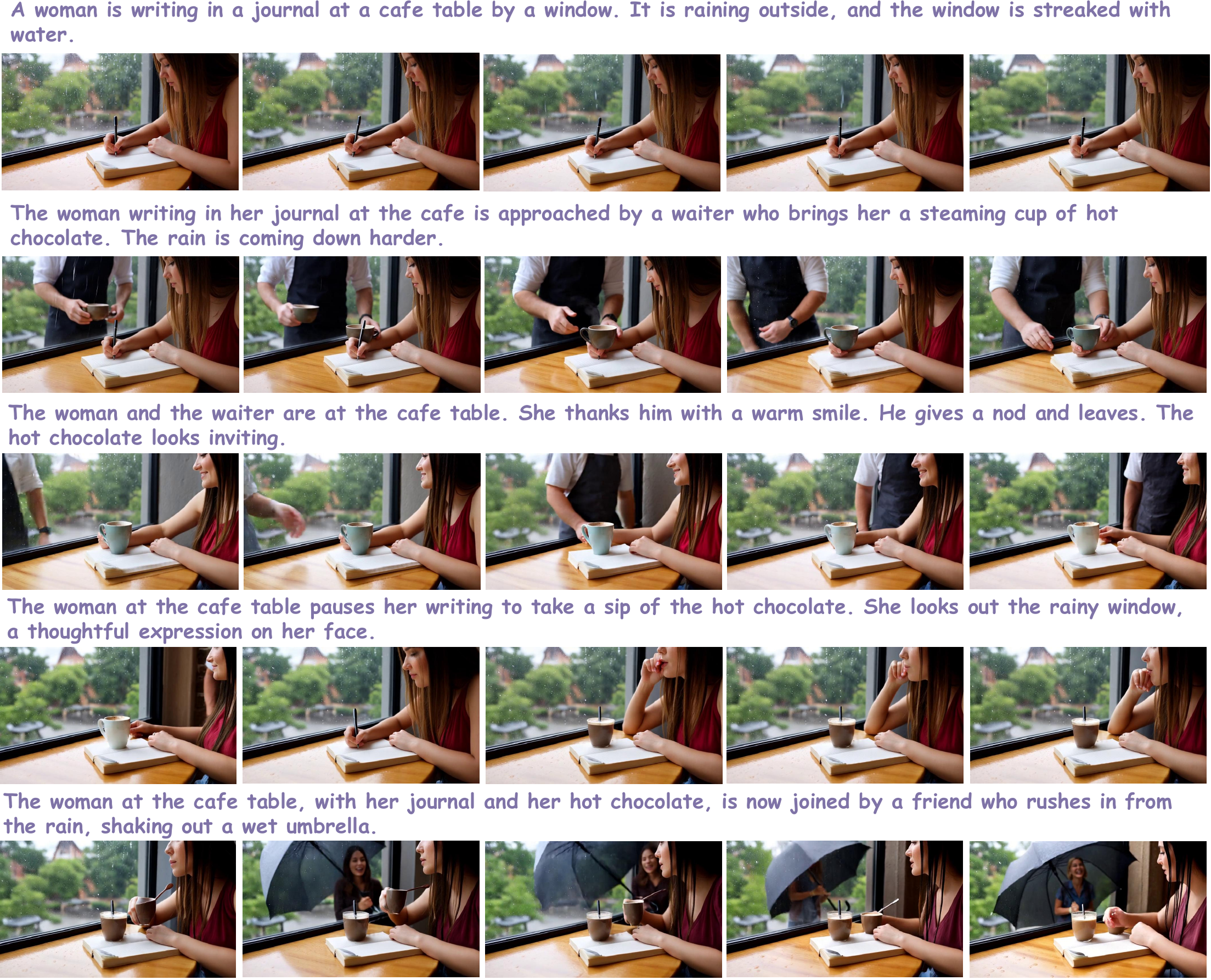}
  \caption{\textbf{Qualitative example of 60-second multi-prompt video generation.} We visualize representative frames generated by SWIFT under a six-prompt sequence.}
  \label{fig:case_study_2}

\end{figure*}
\section{Limitation Analysis}
\label{limitation}

SWIFT is a training-free inference-time framework built upon a pretrained causal video diffusion backbone. Its generation quality therefore remains bounded by the visual fidelity, motion prior, and instruction-following ability of the underlying base model. SWIFT improves semantic transition and memory utilization through lightweight cache injection and prompt-phase adaptive memory allocation, but it does not introduce new visual concepts or correct intrinsic failure modes of the backbone. As a result, difficult cases involving rare objects, fine-grained physical interactions, crowded scenes, or large camera motion may still produce artifacts or inconsistent details. Future work can combine SWIFT with supervised transition data, stronger base video generators, and learned memory scheduling to further improve robustness under complex interactive generation scenarios.
As shown in Figure \ref{fig:case_failure}, the method responds to the newly introduced research partner, but fine-grained facial structure and identity consistency degrade in several frames, as highlighted by the red boxes.
This case illustrates that training-free semantic memory adaptation improves prompt transition and temporal continuity, but cannot fully eliminate backbone-level artifacts under detailed human interaction and compositional prompt changes.
\begin{figure*}[t]
  \centering
  \includegraphics[width=\textwidth]{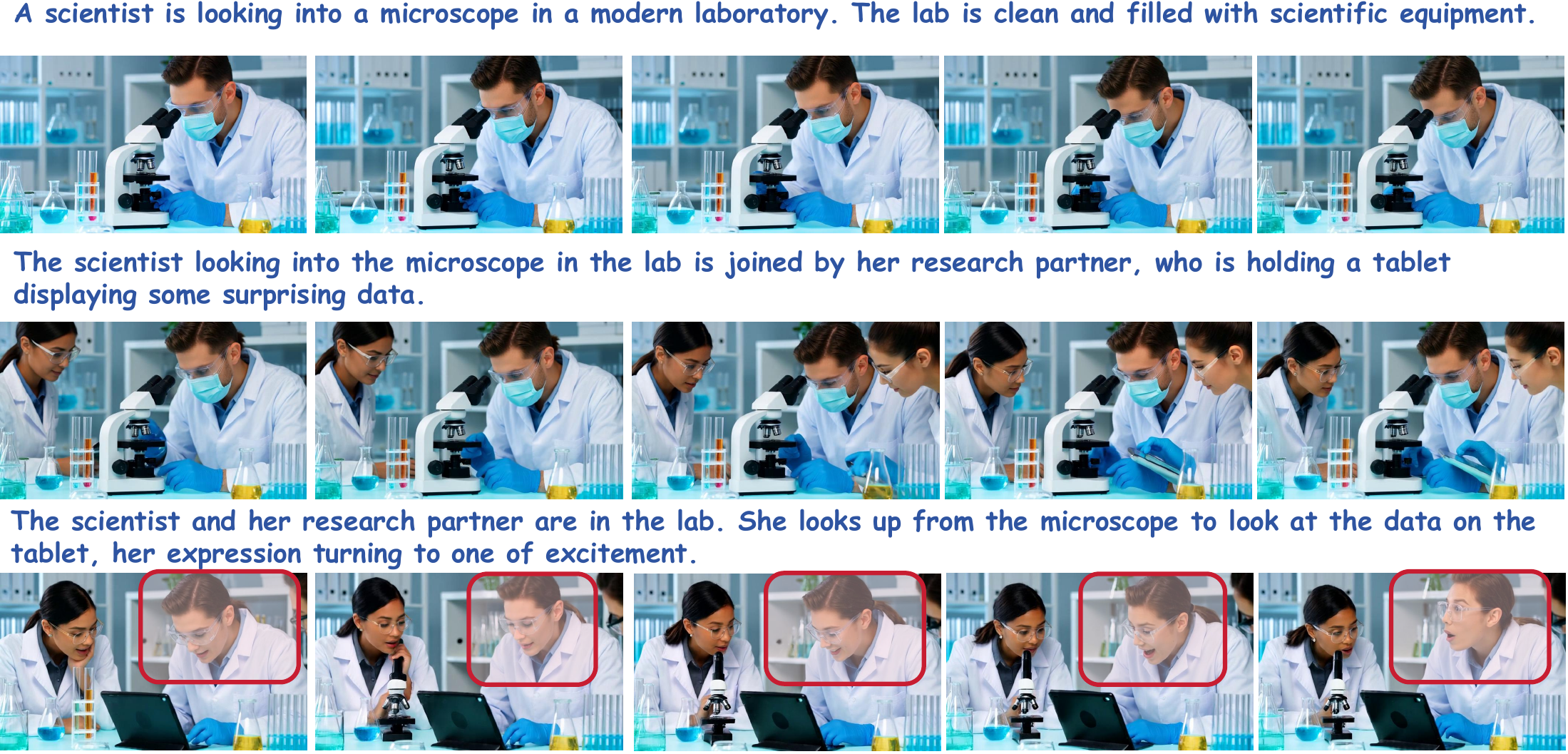}
  \caption{\textbf{Failure case of SWIFT under complex multi-prompt generation.} 
Although SWIFT preserves the laboratory scene and maintains a coherent visual layout across prompt transitions, it can still inherit limitations from the pretrained video diffusion backbone. }
  \label{fig:case_failure}
\end{figure*}


\newpage

\end{document}